%% file: colm2026_conference.tex
\documentclass{article} 
\usepackage[preprint]{colm2026_conference}

\usepackage{microtype}
\usepackage{hyperref}
\usepackage{url}
\usepackage{booktabs}
\usepackage{graphicx}
\usepackage{subfigure}

\usepackage{amsmath}
\usepackage{amssymb}
\usepackage{mathtools}
\usepackage{amsthm}

\usepackage[table]{xcolor}  

\usepackage{graphicx}   
\usepackage{booktabs}    
\usepackage{hyperref}    
\usepackage{array}       
\usepackage{ragged2e}    
\usepackage{cleveref}
\usepackage{wrapfig}
\usepackage{multirow} 

\usepackage[textsize=small]{todonotes}

\usepackage{lineno}

\definecolor{darkblue}{rgb}{0, 0, 0.5}
\hypersetup{colorlinks=true, citecolor=darkblue, linkcolor=darkblue, urlcolor=darkblue}

\title{Skill-CMIB: Multimodal Agent Skill for Consistent Action via Conditional Multimodal Information Bottleneck}

\author{
Zihan Huang$^{1}$\thanks{These authors contributed equally to this work.}, Junda Wu$^{1*}$, Tong Yu$^{2}$, Qianqi Yan$^{3}$, Rohan Surana$^{1}$, \\
\textbf{Uttaran Bhattacharya$^{2}$, Lina Yao$^{4}$, Xin Eric Wang$^{3}$, Julian McAuley$^{1}$}
\\
$^{1}$UC San Diego \quad
$^{2}$Adobe Research \quad
$^{3}$UC Santa Barbara \quad 
$^{4}$University of New South Wales \\
\texttt{\{zih043,juw069,rsurana,jshang,jmcauley\}@ucsd.edu} \\
\texttt{tyu@adobe.com} \quad \texttt{\{qianqi,ericxwang\}@ucsb.edu} \quad \texttt{lina.yao@unsw.edu.au}
}

\theoremstyle{plain}
\newtheorem{assumption}{Assumption}[section]
\newtheorem{definition}[assumption]{Definition}

\newtheorem{lemma}[assumption]{Lemma}

\theoremstyle{remark}

\begin{document}

\ifcolmsubmission
\linenumbers
\fi

\maketitle

\input{content/0_abstract}
\input{content/1_intro}
\input{content/2_related}

\input{content/4_method}

\input{content/6_exp}
\input{content/8_conclusion}

\bibliography{main}
\bibliographystyle{colm2026_conference}

\appendix
\input{content/9_app}

\end{document}

%% file: content/0_abstract.tex
\begin{abstract}
While LLM-based agents increasingly excel at planning and executing long action sequences,
their execution often remains inconsistent across trials, limiting the reliability.
Consolidating agent consistency requires distilling trial-and-error trajectories into reusable skills that preserve task-relevant invariants while discarding trajectory-specific noise.
However, in multimodal settings, the key challenge is not only that useful invariants are distributed across vision and language information,
but that different modalities support different kinds of reusable skill content:
while some agent skills are verbalizable and interpretable, others reside in dense perceptual evidence that text alone cannot capture.
Text-only skills may lose complementary perceptual cues, whereas storing text and perception naively in parallel introduces redundancy and noise.
Existing inference-time methods, such as self-consistency, improve reliability through costly multi-sample decoding,
while existing internalization strategies lack a principled way to separate verbalizable skill content from residual perceptual information.
To address this, we introduce the \textbf{Conditional Multimodal Information Bottleneck (CMIB)}, a principled method for multimodal skill construction.
CMIB begins with a joint bottleneck over multimodal skills and derives an exact sequential decomposition into
(1) a text-stage bottleneck that distills interpretable skill cards, and
(2) a conditional multimodal bottleneck that compresses only the residual information in perception that remains predictive beyond text.
Unlike naive two-stream formulations, CMIB explicitly conditions the multimodal latent on the text skill,
thus structurally reducing cross-modal redundancy and enabling independent control over textual and perceptual compression.
We further instantiate CMIB with a variational objective that makes its conditional decomposition tractable to optimize,
yielding reusable multimodal skills that improve execution stability without incurring multi-sample inference overhead.

\end{abstract}

%% file: content/1_intro.tex
\section{Introduction}

Multimodal LLM agents increasingly operate in environments where successful action depends on both language and perception, 
such as web navigation~\citep{DBLP:conf/nips/DengGZCSWSS23, DBLP:conf/iclr/ZhouX0ZLSCOBF0N24}, 
GUI control~\citep{DBLP:journals/corr/abs-2307-10088, zhang2024mobileenvbuildingqualifiedevaluation,nguyen2025gui}, 
and multimodal decision making and reasoning~\citep{DBLP:conf/icml/DriessXSLCIWTVY23, brohan2023rt2visionlanguageactionmodelstransfer,wu2024personalizedmllm,wu2024visualprompting,wu2025docreact}.
In these settings, the same task can yield markedly different action sequences across trials, even when the underlying policy is unchanged~\citep{shinn2023reflexionlanguageagentsverbal,xia2025sand}.
A common technique to improve model self-consistency at inference time is by 
sampling multiple trajectories and aggregating them through majority voting or related agreement-based procedures~\citep{DBLP:conf/iclr/0002WSLCNCZ23,wu2024decot,wu2025ocean,wu2025ctrls,yu2025explainable}.
While effective in some settings, these approaches rely on substantial decoding cost and do not directly produce a reusable procedure~\cite{wang2025dice} of how certain agent actions are consistent. 
Thus, instead of repeatedly resampling actions, distilling task-relevant invariants from trial-and-error experience into compact multimodal skills can better improve agent consistency.

Agent skills provide a natural interface for reusable control because 
they can condition planning, tool use, and action selection without modifying the underlying task model~\citep{zhang2025equipping, xu2026agentskillslargelanguage, wu2026agent}. 
However, constructing multimodal skills can be challenging, since trial-and-error trajectories contain both reusable procedural structure 
and trajectory-specific noise~\cite{yan2024listitems,li2025commit,wu2025mitigating}, while these signals are distributed unevenly across text and visual information. 
A text-only skill card is interpretable and easy to index in a skill library~\citep{zhang2025equipping,zhang2026memskill,huang2026amps,wang2026scenealign}, but it may discard dense visual evidence that cannot be faithfully verbalized. 
On the other hand, a naive two-stream design that stores textual and visual representations in parallel offers no principled mechanism for deciding what should be captured symbolically 
and what should remain in a latent multimodal channel. 
As a result, the resulting skill can be redundant, noisy, or poorly transferable across task instances~\citep{DBLP:conf/nips/StepputtisCPLBA20}.

To address this problem, we introduce \textbf{Skill-CMIB}, a multimodal skill construction framework grounded in a \emph{Conditional Multimodal Information Bottleneck (CMIB)}. 
We derive a sequential information bottleneck decomposition into two stages: 
a text-stage bottleneck that distills an interpretable skill card, 
and a conditional multimodal bottleneck that compresses only the residual multimodal information that remains predictive beyond the text card. 
This decomposition matches the structure of reusable agent skills. 
The text component serves as the symbolic interface for retrieval, indexing, and explanation,
 while the conditional multimodal component preserves residual perceptual evidence that text alone cannot capture~\citep{DBLP:conf/icml/RadfordKHRGASAM21, DBLP:conf/iclr/GoyalISALBBL19}.

By conditioning the multimodal latent on the text skill, 
CMIB directly penalizes multimodal information that is already explained by the text stream, thereby reducing cross-modal redundancy. 
This yields a representation with three desirable properties:
(1) it remains \textbf{sufficient} for predicting task-relevant outcomes by preserving both the textual procedure and the residual multimodal evidence; 
(2) it is \textbf{minimal} because each stage is explicitly compressed;
(3) and it is \textbf{complementary} that the textual card captures reusable procedural semantics, 
while the multimodal latent focuses on information that is useful and not already contained in the card.
Therefore, CMIB provides a formal answer to what a multimodal skill should store, and in which channel it should be stored.

We further show tractable bounds on information of this information-theoretic view~\cite{wu2022context,liu2025llmcausal}. 
The text stage is realized through prompted skill-card generation under a utility-and-length trade-off, 
producing a compact card that can be stored and reused as a discrete skill artifact. 
The conditional multimodal stage is realized through a variational posterior and prior conditioned on the selected card, 
together with a lightweight projection that fuses the latent into the frozen task model as a soft control prefix \citep{tishby2000information,poole2019variational,mahabadi2021variational,goyal2019infobot,wu2023infoprompt,huang2025traceable}. 
As a result, Skill-CMIB improves agent control without requiring repeated multi-sample decoding at deployment and without updating the backbone task model itself.
We evaluate Skill-CMIB on multimodal agent benchmarks including \texttt{Multimodal-Mind2Web} and \texttt{Mind2Web}, comparing against direct prompting, inference-time self-consistency, and text-only skill cards. The empirical results show that CMIB improves task success and action consistency while offering a more efficient alternative to repeated inference-time sampling. 
We summarize our contributions as follows:
\begin{itemize}
    \item We present an information-theoretic formulation of \emph{multimodal agent skills}, 
    characterizing how reusable procedural structure and task-relevant visual evidence can be organized for consistent multimodal agent behavior.
    \item We propose \textbf{CMIB}, a sequential decomposition of a joint information bottleneck that separates multimodal skill construction into an interpretable text-stage bottleneck and a conditional multimodal bottleneck for residual visual information.
    \item We derive a practical realization, \textbf{Skill-CMIB}, based on tractable variational information bounds, enabling reusable multimodal skill construction for frozen backbone agents without requiring policy parameter updates.
    \item We empirically validate \textbf{Skill-CMIB} on multimodal agent benchmarks, showing improved action consistency and task success compared with direct prompting, inference-time self-consistency, and text-only skill baselines.
\end{itemize}

%% file: content/2_related.tex
\section{Related Works}

\subsection{Agent Skills}

Recent work formalizes \emph{agent skills} as modular, reusable procedures that extend agents beyond atomic tool calls, spanning product-oriented descriptions~\citep{zhang2025equipping,nguyen2025gui,wu2025docreact}, systematic perspectives on skills versus tools~\citep{jiang2026sok}, and surveys on architecture, acquisition, and procedural-memory views of skills~\citep{xu2026agentskillslargelanguage,wu2026agent,huang2025agentic,huang2025surveyfm,wu2024coral}.
A parallel line learns skills or procedural memory from interaction traces via hierarchical memory, reinforcement learning over skill libraries, or non-parametric procedural memory~\citep{fang2025memp,wang2025reinforcement,mi2026procmem,jiang2026xskill}.
Personalized adaptation of LLMs as a relevant context for skill libraries~\cite{zhang2024personalization,xie2025survey,ni2026survey,wang2025self}.
These efforts establish \emph{what} a skill library is and how skills are acquired, but typically expose skills as text or opaque parameters~\cite{wu2024visualprompting,wu2024personalizedmllm} and do not specify how interpretable language and complementary perceptual evidence can be jointly compressed.

\subsection{Behavioral Reliability in LLM-Based Agents}

LLM-based agents exhibit trial-to-trial variability and measurable self-disagreement~\citep{mehta2026agentsdisagreethemselvesmeasuring,xia2025sand,shinn2023reflexionlanguageagentsverbal,wang2025dice}.
The standard inference-time remedy is self-consistency and variants that sample and aggregate multiple outputs~\citep{DBLP:conf/iclr/0002WSLCNCZ23,wang2022self,aggarwal-etal-2023-lets,wang2024soft}, which improves robustness but multiplies decoding cost in sequential settings.
Alternatives \emph{internalize} consistency via post-training~\citep{wu2024decot,samanta2026selfimprovementlanguagemodelsposttraining,kveton2025active,wu2025ocean}, while the information bottleneck~\citep{wu2022context,liu2025llmcausal,tishby2000information} and variational surrogates~\citep{poole2019variational,mahabadi2021variational} instead compress experience into sufficient minimal representations, with instantiations in RL and soft control~\citep{kveton2025active,wu2025incontext,huang2025pluralistic,surana2025mass,mundada2026wsgrpo,DBLP:conf/iclr/GoyalISALBBL19,wu2023infoprompt,huang2025traceable}.
Skill-CMIB follows the latter philosophy for multimodal skills: it avoids repeated voting at deployment and structures compression so a text skill card and a conditional multimodal latent remain complementary~\citep{wang2026scenealign,wu2021deconfounded,li2025commit,DBLP:conf/icml/RadfordKHRGASAM21,DBLP:conf/nips/StepputtisCPLBA20}.

%% file: content/4_method.tex
\section{Skill-CMIB: Multimodal Skills via Conditional Multimodal Information Bottleneck}

We consider a multimodal agentic system that produces trial-and-error trajectories through interaction with an environment. 
We extend the agent skill rollout process in~\Cref{pre:rollout} to multimodality: 
\begin{equation}
    \tau^{(k)} = \bigl(x_t^{(k)},\; m_t^{(k)},\; a_t^{(k)},\; o_t^{(k)},\; f_t^{(k)}\bigr)_{t=1}^{T_k}, \qquad \tau^{(k)}\in \mathcal{B}
    \label{method:rollout}
\end{equation}
where $x_t$ denotes textual prompts, $m_t$ denotes visual pixel information, $a_t$ is the agent's action, $o_t$ is the environment response, and $f_t$ is feedback from the environment. 

\begin{definition}[Multimodal Agent Skill]
Let $X$ and $M$ denote the aggregated textual and multimodal content extracted from a set of rollout trajectories, and let $Y$ denote the supervision target corresponding to the verifiable reward. A \emph{reusable skill} is a structured representation of $(X,M)$ defined as
\begin{equation}\label{method:skill-def}
    S = (c, z) \sim p_{\psi}(S \mid X, M), \qquad c \in \mathcal{C}, \quad z \in \mathbb{R}^d,
\end{equation}
where $c$ is a natural-language \emph{text skill card} used for retrieval, indexing, and coarse control guidance, and $z$ is a \emph{multimodal latent vector} that retains dense perceptual details beyond what can be faithfully verbalized. In this sense, $S=(c,z)$ is a reusable latent skill distilled from $(X,M)$ and intended to retain the task-relevant information needed to predict $Y$.
\end{definition}

The multimodal skill pair $(c, z)$, at deployment, is fused and injected into a frozen task LLM as soft tokens.
Intuitively, the trajectory signal decomposes into:
(1) the \textbf{invariant skill mechanism}, which is a reusable, transferable structure determining what should be done as this skill;
(2) and \textbf{nuisance variation}, which is instance-specific details and noise from both modalities.
The desired multimodal skill should satisfy $Y \perp \!\!  \perp (X, M) \mid S$ (sufficiency), depend primarily on the invariant task skill information,
and expose a text interface for retrieval. Thus, the multimodal skill optimization is to enable the two modalities of $S$ jointly achieve sufficiency and invariance while remaining complementary.

\subsection{Conditional Multimodal Information Bottleneck (CMIB)}\label{sec:cmib}

We extend the information bottleneck in~\Cref{pre:IB} to a joint bottleneck over the multimodal skill $S=(c,z)$ defined in~\Cref{method:skill-def}:
\begin{equation}
\mathcal{L}_{\mathrm{joint}}
= I\bigl((X,M);\,(c,z)\bigr) - \beta\, I\bigl((c,z);\;Y\bigr).
\label{eq:cmib_joint}
\end{equation}
This objective compresses the rollout content $(X,M)$ while preserving the task-relevant information needed to predict the verifiable target $Y$ (illustrated in~\Cref{fig:skill_cmib}). The following lemma shows that \Cref{eq:cmib_joint} admits an exact two-stage factorization, separating a text-stage bottleneck in $c$ from a conditional multimodal bottleneck in $z$ given $c$.

\begin{lemma}[Factorization underlying CMIB]
\label{lem:cmib}
The objective in \Cref{eq:cmib_joint} admits an exact decomposition into a text-stage term involving $c$ and a conditional multimodal term involving $z$ given $c$. Motivated by this factorization, we define the Conditional Multimodal Information Bottleneck (CMIB) by introducing stage-specific trade-off coefficients:
\begin{equation} \label{eq:cmib_objective}
\mathcal{L}_{\mathrm{CMIB}}
=
\underbrace{\Bigl[I\bigl((X,M);c\bigr)-\beta_c\,I(c;Y)\Bigr]}_{\text{Text bottleneck}}
+
\underbrace{\Bigl[I\bigl((X,M);z\mid c\bigr)-\beta_z\,I(z;Y\mid c)\Bigr]}_{\text{Conditional multimodal bottleneck}}.
\end{equation}
When $\beta_c=\beta_z=\beta$, \Cref{eq:cmib_objective} reduces to the exact factorization of \Cref{eq:cmib_joint}.
\end{lemma}
The proof is a direct consequence of the chain rule of mutual information as in \Cref{app:proof_cmib}.

\begin{wrapfigure}{r}{0.36\textwidth}
    \centering
    \includegraphics[width=0.36\textwidth]{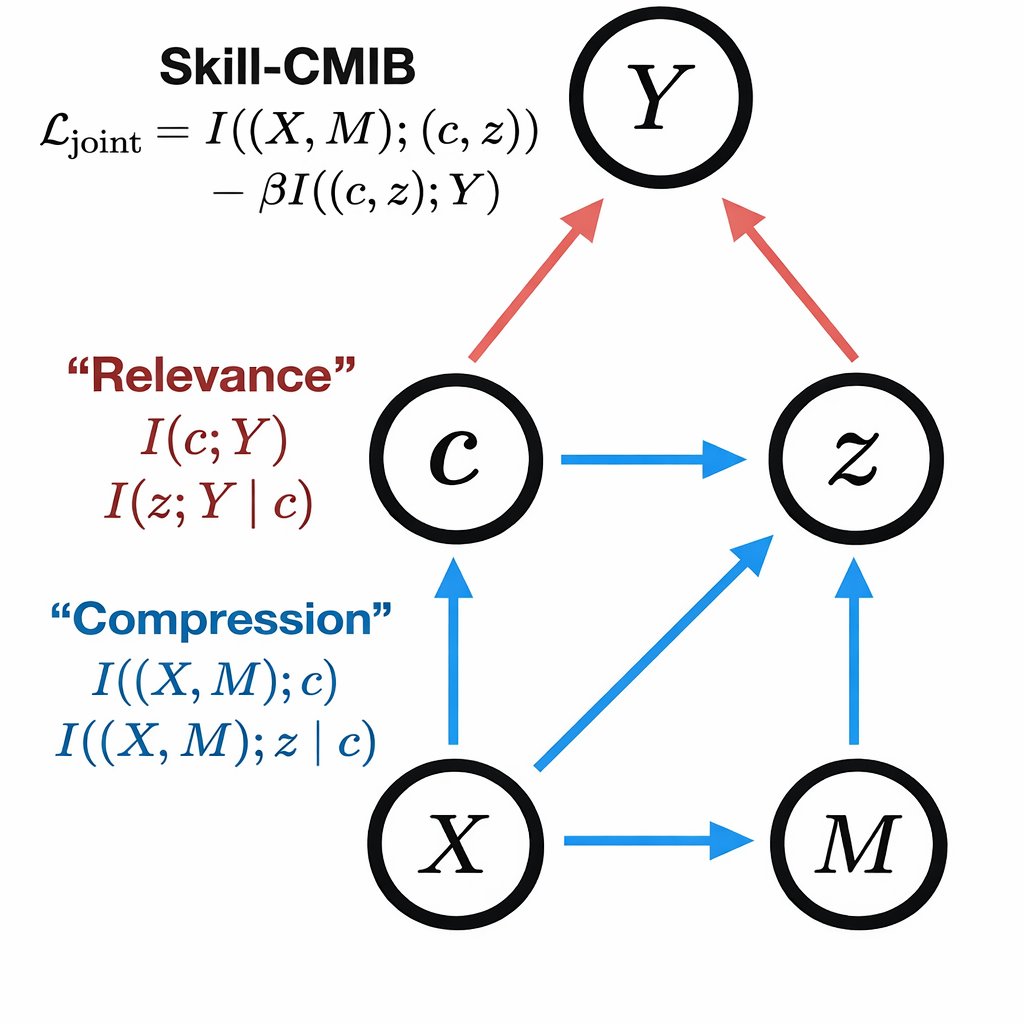}
    \caption{Skill-CMIB illustration.}
    \label{fig:skill_cmib}
\end{wrapfigure}

\Cref{eq:cmib_objective} directly encodes the three design requirements of multimodal skill construction:
\textbf{Sufficiency} is captured by the relevance terms $I(c;Y)$ and $I(z;Y\mid c)$, which preserve the information needed to predict $Y$, with the conditional term measuring the residual predictive value beyond the text card. 
\textbf{Minimality} is imposed by the compression terms $I((X,M);c)$ and $I((X,M);z\mid c))$, which suppress instance-specific noise, with the latter compressing only information not already encoded in $c$. 
\textbf{Complementarity} follows from conditioning the multimodal latent on $c$. 
The text card $c$, optimized through the unconditional bottleneck $I((X,M);c))-\beta_c I(c;Y)$, 
captures reusable semantics of the rollout and serves as the interface for indexing, retrieval, and explanation in skill libraries.
In turn, the conditional bottleneck $I((X,M);z\mid c))-\beta_z I(z;Y\mid c)$ drives $z$ to retain only residual multimodal evidence that remains useful once $c$ is known \citep{li2026skillsbench,xu2026agent}.

CMIB also suggests a natural representation strategy for the agentic workflow. 
At construction time, $c$ and $z$ are stored separately because they are produced by different encoders and serve different roles: 
$c$ provides the symbolic interface for skill management, whereas $z$ preserves perceptual detail that is inefficient to verbalize. 
At inference time, the multimodal information $z$ is projected in to skill prefix,
\begin{equation}
u = g_\omega(z),
\label{eq:cmib_fusion}
\end{equation}
where $g_\omega$ is a lightweight projection module. 
This separation makes the bottleneck traceable: the statistics of the text and the conditional multimodal stream can be monitored independently, 
while the frozen model still receives unified control signals at inference time. 

We now instantiate this view through \textbf{tractable bounds on information} in a multimodal LLM agent. 
Because the information bottleneck terms in \Cref{eq:cmib_objective} are not directly tractable in practice~\citep{tishby2000information, poole2019variational, mahabadi2021variational, huang2025traceable}, we replace them with computable surrogates that track each component of CMIB~\citep{goyal2019infobot, wu2023infoprompt}. 
We begin with the \textbf{text bottleneck}, then derive the \textbf{conditional multimodal bottleneck}, and combine the two into a unified training objective.

\subsection{Tractable Bound for the Text Bottleneck}
The text bottleneck aims to extract a compact skill card $c$ that preserves the task-relevant procedural content of the rollout bundle
while remaining short enough to act as the symbolic interface of the skill library. Formally, the text-stage term of CMIB is
\begin{equation}
    \mathcal{L}_{c} = I((X,M);c)-\beta_c I(c;Y).
    \label{eq:text_ib_orig}
\end{equation}
Here, the compression term discourages unnecessary dependence of $c$ on the rollout content,
while the relevance term encourages $c$ to retain information predictive of the verifiable target $Y$.
Since $c$ is generated by a frozen LLM rather than a trainable probabilistic encoder, we do not optimize \Cref{eq:text_ib_orig} directly. Instead, we define the feasible candidate set
\begin{equation}
\label{eq:text_ib_transformed}
\mathcal{J}_{c}(c;\mathcal B)
=
|c|-\beta_c\,\widehat U(c;\mathcal B),
\quad
c^*
=
\arg\min_{c\in\mathcal{C}_{L_c}(X)}
\mathcal{J}_{c}(c;\mathcal B),
\end{equation}
where $|c|$ denotes the length of $c$ (e.g., number of tokens) and $\widehat U(c;\mathcal B)$ is the task-specific utility score of card $c$ evaluated on the rollout bundle $\mathcal B$ in \Cref{method:rollout}. Let $\mathcal{C}_{L_c}(X)\subseteq\mathcal{C}$ denote the set of candidate skill cards obtainable from the aggregated rollout text $X$ under length budget $L_c$. In this formulation, the objective $\mathcal{J}_{c}$ tracks the relevance term $I(c;Y)$, while restricting $c$ to $\mathcal{C}_{L_c}(X)$ enforces the compression budget associated with $I((X,M);c)$.

Because \Cref{eq:text_ib_transformed} optimizes over discrete skill cards produced by a frozen LLM, we realize it through prompting rather than direct optimization. 
Let $\Pi_c(X,L_c)$ denote the prompt constructor that maps the aggregated rollout text $X$ and the length budget $L_c$ to a skill-card generation prompt,
and let
\begin{equation}
c^* \sim \pi_{\mathrm{sc}}(\cdot \mid X,L_c)
\;=\;
\pi_{\mathrm{sc}}(\cdot \mid \Pi_c(X,L_c)).
\label{eq:text_ib_prompt}
\end{equation}
be the induced generation distribution under the frozen LLM. 
The prompt $\Pi_c(X,L_c)$ is instantiated by a \textbf{progressive summarization pipeline}: trajectory-level evidence is first extracted from each rollout, 
then aggregated across the $K$ rollouts, and finally formatted into a structured card. 
The selected output $c^*$ from \Cref{eq:text_ib_prompt} is thus the tractable approximation to the original text bottleneck in \Cref{eq:text_ib_orig},
and serves as the discrete interface used later for retrieval, indexing, and explanation.

\subsection{Tractable Bound for the Conditional Multimodal Bottleneck}

Since the text stage has already produced the discrete component $c^*$, the remaining problem is to construct the residual multimodal latent $z$ conditioned on that fixed skill card. 
Under this realization, the conditional multimodal term in \Cref{eq:cmib_objective} becomes
\begin{equation}
\mathcal{L}_{z}
=
I((X,M);z\mid c^*)-\beta_z I(z;Y\mid c^*).
\label{eq:mm_ib_orig}
\end{equation}
The following lemma shows that this objective admits a tractable variational surrogate based on a text-conditioned posterior encoder and prior.

\begin{lemma}[Variational surrogate of the conditional multimodal bottleneck]
\label{lem:mm_ib_variational}
Fix a text card $c^*$, and suppose the latent variable $z$ is generated by the encoder
$q_\theta(z\mid M,c^*)$, so that $Z \perp X \mid (M,c^*)$.
Let $g_\omega$ be the projection map introduced in \Cref{eq:cmib_fusion}, which maps the multimodal latent $z$ into the control space of the frozen task model $\pi_{\mathrm{tsk}}$.
For any conditional prior $r_\phi(z\mid c^*)$, define
\begin{equation}
\mathcal{J}_{z}(\theta,\phi;c^*)
=
\mathbb{E}_{\substack{(M,Y)\sim p(\cdot,\cdot\mid c^*)\\ z\sim q_\theta(\cdot\mid M,c^*)}} \left[
\log \frac{q_\theta(z\mid M,c^*)}{r_\phi(z\mid c^*)}
-
\beta_z \log \pi_{\mathrm{tsk}} \left(
Y \mid [\,g_\omega(z);\; c^*;\; \mathcal{B}\,]
\right)
\right].
\label{eq:mm_ib_transformed}
\end{equation}
Then
\begin{equation}
\mathcal{L}_{z}
\;\le\;
\mathcal{J}_{z}(\theta,\phi;c^*)-\beta_z H(Y\mid c^*).
\label{eq:mm_ib_bound}
\end{equation}
Equivalently, up to the additive constant $\beta_z H(Y\mid c^*)$, the surrogate $\mathcal{J}_{z}(\theta,\phi;c^*)$ upper-bounds the original conditional multimodal bottleneck.
\end{lemma}
By \Cref{lem:mm_ib_variational}, the KL term controls the conditional compression term, while the predictive log-likelihood term provides a variational lower bound on the conditional relevance term up to the additive constant $H(Y\mid c^*)$. The full proof is deferred to \Cref{app:mm_ib_variational_proof}. The posterior and prior are parameterized as

\begin{equation}
\begin{aligned}
q_\theta(z\mid M,c^*) &= \mathcal N\bigl(\mu_\theta(M,c^*),\,\Sigma_\theta(M,c^*)\bigr), \\
r_\phi(z\mid c^*) &= \mathcal N\bigl(\mu_\phi(c^*),\,\Sigma_\phi(c^*)\bigr)
\end{aligned}
\label{eq:mm_ib_latent}
\end{equation}
where $\Sigma_\theta$ and $\Sigma_\phi$ are diagonal, $\sigma_\theta(M,c^*)$ is the elementwise standard-deviation vector, 
the posterior conditions multimodal rollout features on the fixed text card, 
and the prior represents the default multimodal expectation induced by $c^*$ alone.
During training, we use the standard reparameterization

\begin{equation}
\begin{aligned}
z &= \mu_\theta(M,c^*) + \sigma_\theta(M,c^*) \odot \epsilon, \\
\epsilon &\sim \mathcal N(0,I), \qquad
\Sigma_\theta(M,c^*) = \mathrm{diag}\!\bigl(\sigma_\theta(M,c^*) \odot \sigma_\theta(M,c^*)\bigr)
\end{aligned}
\label{eq:mm_ib_reparam}
\end{equation}

Finally, the realized latent $z$ is fused with the text card through the control map already introduced in \Cref{eq:cmib_fusion}. Concretely, the projected latent $g_\omega(z)$ is prepended together with the fixed card $c^*$ and the rollout bundle $\mathcal{B}$ to the frozen task model $\pi_{\mathrm{tsk}}$, so that the prediction term in \Cref{eq:mm_ib_transformed} measures how much additional task-relevant information remains in $z$ once the symbolic interface $c^*$ has already been fixed.

\subsection{Overall CMIB Objective}

Combining the text-stage surrogate with the conditional multimodal surrogate yields the overall tractable realization of CMIB. Since the textual skill card is produced by discrete prompting rather than direct gradient-based optimization, we first the best textual skill card $c^*$ and then optimize the continuous multimodal stage conditioned on the selected card. The resulting overall objective is
\begin{equation}
\widetilde{\mathcal L}_{\mathrm{CMIB}}(\theta,\phi;c^*)
=
\mathcal{J}_{c}(c^*;\mathcal B)
+
\mathcal{J}_{z}(\theta,\phi;c^*).
\label{eq:cmib_surrogate}
\end{equation}
~\Cref{eq:cmib_surrogate} is the trackable realization of the proposed information bottleneck~\Cref{eq:cmib_objective} that makes explicit how each information term in CMIB is realized in practice: the text-stage bottleneck is tracked by card length and task utility, while the conditional multimodal bottleneck is realized by the variational KL term and the predictive log-likelihood term. All trainable components are confined to the posterior encoder $q_\theta$, the prior $r_\phi$, and the projection map $g_\omega$, while the frozen task model $\pi_{\mathrm{tsk}}$ is never updated.

The final multimodal skill is a concrete realization of  \Cref{method:skill-def}: after the text stage selects $c^*$ from \Cref{eq:text_ib_prompt}, we instantiate the multimodal skill as
\begin{equation}
S^*=(c^*,z^*), \qquad
c^* \sim \pi_{\mathrm{sc}}(\cdot\mid \Pi_c(X,L_c)), \quad
z^* \sim q_\theta(\cdot\mid M,c^*),
\label{eq:cmib_final_skill}
\end{equation}
which is the realized form of $p_\psi(S\mid X,M)$ under the two-stage CMIB construction in~\Cref{sec:cmib}.
In this way, the text card $c^*$ supplies an interpretable and retrievable procedural interface, while the latent $z^*$ injects complementary multimodal evidence that cannot be faithfully compressed into text alone. The task model thus consumes the learned skill through both an explicit symbolic prompt and a fused soft control state.

%% file: content/6_exp.tex
\section{Experiments}

In this section, we conduct experiments to evaluate the proposed CMIB framework, aiming to validate its theoretical properties and demonstrate its practical advantages in multimodal agentic environments.
Our primary objective is to verify whether CMIB can effectively achieve the information-theoretic goals of sufficiency, minimality, and complementarity, thereby leading to enhanced task performance and inference efficiency.
This investigation is guided by the following research questions:

\begin{itemize}
\item \textbf{RQ1: CMIB Effectiveness.} To what extent does CMIB improve task success rate and trajectory consistency of multimodal agents compared to state-of-the-art baselines?
\item \textbf{RQ2: Action Consistency.} How does CMIB affect action-level consistency across repeated trials compared with vanilla inference and self-consistency?
\item \textbf{RQ3: Ablation Study.} Does the conditional latent vector successfully capture residual perceptual information, and does it exhibit the theoretical collapse property when text is sufficient?
\item \textbf{RQ4: Inference Efficiency and Computational Cost.} Does CMIB mitigate the "prohibitively expensive" nature of sequential action generation compared to inference-time sampling?
\end{itemize}

\textbf{Dataset.} We evaluate CMIB on two benchmarks: \texttt{Multimodal-Mind2Web}~\citep{DBLP:conf/nips/DengGZCSWSS23} and \texttt{ScreenSpot}~\citep{DBLP:conf/acl/ChengSCX0Z024}. 
Mind2Web provides real-world web tasks spanning multiple domains and websites.
ScreenSpot offers a fine-grained grounding benchmark across mobile, desktop, and web platforms, featuring both text and icon/widget elements.

\textbf{Metrics.} For performance evaluation, we follow standard metrics from the Multimodal-Mind2Web benchmark~\citep{DBLP:conf/nips/DengGZCSWSS23, DBLP:conf/acl/PahujaLRGMW0A25}, including element accuracy (Ele. Acc), operation F1 (Op. F1), step success rate (Step SR), and task success rate (SR). To assess action stability, we introduce Step Consistency (StepCons), which measures the pairwise agreement of normalized actions across repeated trials.

\textbf{Model.} We evaluate CMIB framework based on \texttt{Qwen2.5-VL-7B-Instruct}~\citep{qwen2.5-VL}. The CMIB augments this backbone with a skill library. Specifically, a \texttt{Qformer}~\citep{DBLP:journals/pami/ZhangZXT24} and \texttt{MLP} module are used to encode multimodal trajectories into a latent vector, which is then decoded into soft prompts to guide the agent. We compare Agent with CMIB against: (1) \emph{Vanilla Agent} (no skill injection), (2) \emph{Text-Only Skill Card}, and (3) \emph{Self-Consistency}~\citep{wang2022self} with up to $K{=}5$ multi-sample decoding. To ensure a fair comparison, all methods are evaluated on the same splits and candidate sets under identical prompting.

\subsection{(RQ1) CMIB Effectiveness}

We evaluate CMIB on \texttt{Multimodal-Mind2Web} for web navigation and \texttt{ScreenSpot} for GUI grounding. Table~\ref{tab:multimodal_mind2web} compares CMIB against a wide range of baselines, including in-context learning (ICL), supervised fine-tuning (SFT), and methods combining data synthesis with SFT. Without task-specific tuning on LLMs, CMIB consistently outperforms the Qwen2.5-VL-7B-instruction baseline across all splits, achieving evident gains in Step SR. This demonstrates that the conditional multimodal information bottleneck effectively extracts complementary visual-textual cues into skill library. 

\begin{table}[ht]
\label{tab:multimodal_mind2web}
\centering
\resizebox{\textwidth}{!}{%
\begin{tabular}{l c c c c c c c c c c}
\toprule
\multicolumn{1}{c}{\textbf{Method}} &
\multicolumn{3}{c}{\textbf{Cross-Task}} &
\multicolumn{3}{c}{\textbf{Cross-Website}} &
\multicolumn{3}{c}{\textbf{Cross-Domain}} &
\multicolumn{1}{c}{\textbf{Avg. Step SR}} \\
\cmidrule(lr){2-4} \cmidrule(lr){5-7} \cmidrule(lr){8-10}
& \textbf{Ele. Acc} & \textbf{Op. F1} & \textbf{Step SR} &
\textbf{Ele. Acc} & \textbf{Op. F1} & \textbf{Step SR} &
\textbf{Ele. Acc} & \textbf{Op. F1} & \textbf{Step SR} & \\
\midrule
\multicolumn{11}{c}{\textbf{In-Context Learning}} \\
GPT-3.5~\citep{DBLP:conf/acl/PahujaLRGMW0A25} & 19.4 & 59.2 & 16.8 & 14.9 & 56.5 & 14.1 & 25.2 & 57.9 & 24.1 & 18.3 \\
GPT-4~\citep{DBLP:conf/acl/PahujaLRGMW0A25} & 40.8 & 63.1 & 32.3 & 30.2 & 61.0 & 27.0 & 35.4 & 61.9 & 29.7 & 29.7 \\
SeeAct~\citep{DBLP:conf/icml/ZhengGK0024} & 46.4 & 73.4 & 40.2 & 38.0 & 67.8 & 32.4 & 42.4 & 69.3 & 36.8 & 36.5 \\
\midrule
\multicolumn{11}{c}{\textbf{Supervised Fine-Tuning}} \\

Classification~\citep{DBLP:conf/nips/DengGZCSWSS23} & 26.8 & -- & -- & 21.6 & -- & -- & 24.5 & -- & -- & -- \\
Generation~\citep{DBLP:conf/nips/DengGZCSWSS23}   & 20.2 & 52.0 & 17.5 & 13.9 & 44.7 & 11.0 & 14.2 & 44.7 & 11.9 & 13.5\\
Explorer-4B~\citep{DBLP:conf/acl/PahujaLRGMW0A25} & 48.1 & 88.0 & 44.8 & 49.1 & 87.2 & 45.0 & 46.9 & 87.7 & 44.6 & 44.8 \\
Explorer-7B~\citep{DBLP:conf/acl/PahujaLRGMW0A25} & 51.8 & 88.0 & 48.3 & 56.3 & 89.7 & 52.0 & 50.9 & 88.9 & 48.1 & 49.5 \\
\midrule
\multicolumn{11}{c}{\textbf{Data Synthesis+Supervised Fine-Tuning}} \\
SeeClick-9.6B~\citep{DBLP:conf/acl/ChengSCX0Z024} & 26.3 & 86.2 & 23.7 & 21.9 & 82.9 & 18.8 & 22.1 & 84.1 & 20.2 & 20.9 \\
EDGE-9.6B~\cite{chen2024edgeenhancedgroundedgui} & -- & -- & 30.0 & -- & -- & 21.1 & -- & -- & 22.4 & 24.5 \\
MiniCPM-GUI-3.1B~\citep{chen2025guicoursegeneralvisionlanguage} & 23.8 & 86.8 & 20.8 & 20.3 & 81.7 & 17.3 & 17.9 & 74.5 & 14.6 & 17.6 \\
Scribe-Agent-L.-32B~\citep{shen2024scribeagentspecializedwebagents} & 38.0 & 52.9 & 35.6 & 34.1 & 52.7 & 32.5 & 39.4 & 54.7 & 37.3 & 35.1 \\
AgentTrek-7B~\citep{xu2025agenttrekagenttrajectorysynthesis} & 60.8 & 88.9 & 55.7 & 57.6 & 88.1 & 51.4 & 56.0 & 87.5 & 52.6 & 53.2 \\
Explorer-4B~\citep{DBLP:conf/acl/PahujaLRGMW0A25} & 53.4 & 88.1 & 50.7 & 55.6 & 89.5 & 51.4 & 49.8 & 88.8 & 47.2 & 49.8 \\
Explorer-7B~\citep{DBLP:conf/acl/PahujaLRGMW0A25} & 56.5 & 90.3 & 53.2 & 60.5 & 90.7 & 56.7 & 55.7 & 90.4 & 53.0 & 54.3 \\
\midrule
\multicolumn{11}{c}{\textbf{Skill Based}} \\
Qwen2.5-VL-7B-Inst. & 51.9 & 63.5 & 33.4 & 44.7 & 57.5 & 25.5 & 46.8 & 66.3 & 32.4 & 30.4 \\
\textbf{Skill-CMIB} & 57.0 & 59.8 & 39.0 & 55.4 & 60.8 & 34.1 & 57.4 & 69.5 & 42.9 & 38.7 \\
\bottomrule
\end{tabular}
}
\caption{Results across different settings on Multimodal-Mind2Web~\citep{DBLP:conf/nips/DengGZCSWSS23}.} 
\end{table}

Among other training free methods, CMIB ourperforms ICL baselines including SeeAct and GPT-4, highlighting the advantage of structured multimodal skill modeling over pure prompting. While models trained with additional synthetic data such as AgentTrek-7B, Explorer-7B achieve higher absolute Step SR, CMIB remains competitive.

Table \ref{tab:screenspot} further shows that CMIB achieves superior or competitive performance on \texttt{ScreenSpot}, particularly in challenging categories like Mobile Icon/Widget and Web Text, reinforcing its ability to leverage complementary multimodal cues, demonstrate the effectiveness of CMIB.

\begin{table}[htbp]
\label{tab:screenspot}
\centering
\scriptsize
\vspace{-1em}
    \begin{tabular}{lccccccc}
    \toprule
    \multirow{2.5}{*}{\textbf{LVLMs (Size)}} & \multicolumn{2}{c}{\textbf{Mobile}} & \multicolumn{2}{c}{\textbf{Desktop}} & \multicolumn{2}{c}{\textbf{Web}} & \multirow{2.5}{*}{\textbf{Avg.}} \\
    \cmidrule(lr){2-3} \cmidrule(lr){4-5} \cmidrule(lr){6-7}
    & Text & Icon/Widget & Text & Icon/Widget & Text & Icon/Widget & \\
    \midrule
    MiniGPT-v2-7B & 8.4 & 6.6 & 6.2 & 2.9 & 6.5 & 3.4 & 5.7 \\
    Qwen-VL-9.6B & 9.5 & 4.8 & 5.7 & 5.0 & 3.5 & 2.4 & 5.2 \\
    GPT-4V & 22.6 & 24.5 & 20.2 & 11.8 & 9.2 & 8.8 & 16.2 \\
    Fuyu-8B & 41.0 & 1.3 & 33.0 & 3.6 & 33.9 & 4.4 & 19.5 \\
    CogAgent-18B & 67.0 & 24.0 & 74.2 & 20.0 & 70.4 & 28.6 & 47.4 \\
    SeeClick-9.6B & 78.0 & 52.0 & 72.2 & 30.0 & 55.7 & 32.5 & 53.4 \\
    \midrule
    Qwen2.5-VL-7B-Inst. & 60.8 & 53.3 & 57.7 & 42.1 & 51.7 & 46.6 & 53.0 \\
    \textbf{Skill-CMIB} & 64.8 & 65.1 & 54.1 & 46.4 & 60.0 & 49.5 & 57.9 \\
    \bottomrule
    \end{tabular}
\caption{Action Success Rate (\%) compared with baselines on ScreenSpot~\citep{DBLP:conf/nips/DengGZCSWSS23}.}
\vspace{-1.5em}

\end{table}

\subsection{(RQ2) Action Consistency}

Task success alone does not capture the stability of intermediate decisions. To evaluate this, we run each setting for N=3 independent repeats on the same evaluation split and compute \textbf{Step Consistency} (StepCons) to heuristically evaluate action consistency.
\begin{equation}
C_{\text{step}}^{(i,j)}=\frac{1}{|\mathcal{S}_{ij}|}\sum_{t\in\mathcal{S}_{ij}}\mathbf{1}\!\left[\tilde{a}_{t}^{(i)}=\tilde{a}_{t}^{(j)}\right], \qquad
\text{StepCons}=\frac{1}{\binom{N}{2}}\sum_{i<j} C_{\text{step}}^{(i,j)}.
\end{equation}
As shown in Table \ref{tab:mind2web_consistency}, CMIB achieves a substantially higher StepCons score compared to the best self-consistency baseline and the vanilla agent. This indicates that CMIB leads to significantly more stable and reliable action selection.

As $k$ increases from 1 to 5, both Step SR and StepCons improve consistently across splits. With average Step SR rises from 33.77\% to 35.98\%, and StepCons from 0.0866 to 0.1789. This trend aligns with observations in \citep{mehta2026agentsdisagreethemselvesmeasuring}, where larger $k$ reduces variance in multi-step reasoning and the more consistent trajectory brings higher success rate. While self-consistency with $k=5$ achieves comparable or even slightly better element accuracy and Step SR on certain splits, its StepCons remains substantially lower than CMIB, 0.1789 and 0.4144. This gap highlights that CMIB not only maintains competitive task performance but also yields far more consistent intermediate actions, underscoring its effectiveness in improving agent inference stability.

\begin{table}[htbp]
\centering
\scriptsize
\begin{tabular}{llcccccc}
\toprule
\textbf{Setting} & \textbf{Metric} & \textbf{SC (k=1)} & \textbf{SC (k=2)} & \textbf{SC (k=3)} & \textbf{SC (k=4)} & \textbf{SC (k=5)} & \textbf{CMIB (Full)} \\
\midrule
\multirow{4}{*}{\textbf{Cross-Task}}
& Ele. Acc  & 52.98 & 52.98 & 54.43 & 56.23 & 56.78 & 51.35 \\
& Op. F1    & 67.13 & 67.13 & 67.94 & 68.27 & 68.12 & 61.08 \\
& Step SR   & 38.39 & 38.39 & 40.51 & 41.49 & 41.48 & 39.16 \\
& StepCons  & 11.33 & 11.33 & 11.00 & 12.00 & 17.33 & 43.67 \\
\midrule
\multirow{4}{*}{\textbf{Cross-Website}}
& Ele. Acc  & 38.77 & 38.77 & 39.80 & 40.63 & 41.78 & 44.68 \\
& Op. F1    & 61.78 & 61.78 & 61.10 & 62.73 & 64.94 & 62.84 \\
& Step SR   & 24.48 & 24.48 & 24.87 & 25.31 & 26.54 & 28.14 \\
& StepCons  & 6.33  & 6.33  & 10.00 & 12.00 & 16.00 & 42.33 \\
\midrule
\multirow{4}{*}{\textbf{Cross-Domain}}
& Ele. Acc  & 57.02 & 57.02 & 58.02 & 59.23 & 59.88 & 69.88 \\
& Op. F1    & 67.71 & 67.71 & 67.14 & 68.70 & 68.80 & 75.13 \\
& Step SR   & 38.45 & 38.45 & 38.25 & 40.07 & 39.91 & 49.42 \\
& StepCons  & 8.33  & 8.33  & 12.00 & 16.33 & 20.33 & 38.33 \\
\midrule
\textbf{Average} & Avg. Step SR   & 33.77 & 33.77 & 34.54 & 35.62 & 35.98 & 38.91 \\
\textbf{Average} & Avg. StepCons  & 8.66  & 8.66  & 11.00 & 13.44 & 17.89 & 41.44 \\
\bottomrule
\end{tabular}%
\caption{Model Performance on Multimodal-Mind2Web. Results present as percentages.
}
\vspace{-1.5em}
\label{tab:mind2web_consistency}
\end{table}

\subsection{(RQ3) Ablation Study}

To validate the theoretical underpinnings of CMIB, we perform an ablation study on the multimodal stage. We compare full CMIB model against variants using only text cards or independent $c$ and $z$ inputs, as well as Qwen2.5-VL-7B by measuring their Step Success Rate and Information Redundancy between c and z with the average KL divergence of $q_{\theta}(z\mid M,c^*)$ and $r_{\phi}(z\mid c^*)$. Table \ref{tab:ablation_avg_step_redundancy_new} shows that CMIB achieves the highest Step SR while also exhibiting much lower information redundancy $I(c;z)$, indicating that its ability to encourage learning a minimal representation.

We further dissect each component’s contribution. Removing the redundancy constraint between $z$ and $c$ results in a sharp rise in redundancy $I(c;z)$ and a slight drop in Step SR to 39.18. This confirms that the constraint not only encourages a minimal representation of $z$ conditioned on $c$, but also helps $z$ capture complementary visual information beyond the text card. Then, we omit $z$ entirely (\textit{Text Card $c$ only}), which further reduces Step SR to 37.95, indicating that $z$ encodes additional perceptual skill information not covered by $c$. Finally, removing the entire CMIB skill library (\textit{No skill}) leads to the largest performance drop to 30.62 in Step SR, highlighting the effectiveness of the multimodal skill library. These observations confirm that each CMIB component contributes meaningfully to performance.

\begin{table}[h]
\centering
\small
\label{tab:ablation_avg_step_redundancy_new}
\begin{tabular}{lcc}
\toprule
\textbf{Method} & \textbf{Avg. Step SR (\%)} & \textbf{Avg. Redundancy $I(c;z)$} \\
\midrule
No skill (Qwen only) & 30.62 & -- \\
Text Card $c$ only & 37.95 & -- \\
Independent ($c, z$) & 39.18 & 1.69 \\
\textbf{Skill-CMIB ($c, z \mid c$)} & 41.66 & 0.18 \\
\bottomrule
\end{tabular}
\caption{Aggregated results (average over three splits) on Mind2Web, where \textit{No skill (Qwen only)} and \textit{Text Card c only} setting does not have $I(c;z)$.}
\end{table}

\subsection{(RQ4) Inference Efficiency and Computational Cost}

We analyze the trade-off between performance and computational cost. As shown in Table \ref{tab:efficiency}, CMIB offers a significant efficiency advantage. While Self-Consistency improves performance at the cost of a $K \times$ inference overhead, CMIB achieves better or comparable results with minimal additional inference cost, based on a lightweight Q-former and MLP projector and reusable multimodal skill library without increasing the main model's latency.

\begin{figure}[htbp]
    \centering
    \begin{minipage}[t]{0.52\textwidth}
        \centering
        \vspace{-10em}
        \resizebox{\textwidth}{!}{%
        \begin{tabular}{lccc}
            \toprule
            \textbf{Method} & \textbf{Step SR (\%)} & \textbf{Latency (ms/step)} \\
            \midrule
            Qwen2.5-VL-7B & 38.39 & 3483.2 \\
            \textbf{Skill-CMIB} & 39.16 & 3510.4 \\
            Self-Consistency ($K=3$) & 40.51 & 10599.6 \\
            Self-Consistency ($K=5$) & 41.48 & 17680.6 \\
            \bottomrule
        \end{tabular}%
        }
        \caption{Efficiency analysis. Skill-CMIB achieves higher task success rate with significantly lower inference latency, demonstrating a favorable performance–cost trade-off.}
        \label{tab:efficiency}
    \end{minipage}
    \hfill
    \begin{minipage}[t]{0.36\textwidth}
        \centering
        \includegraphics[width=\linewidth]{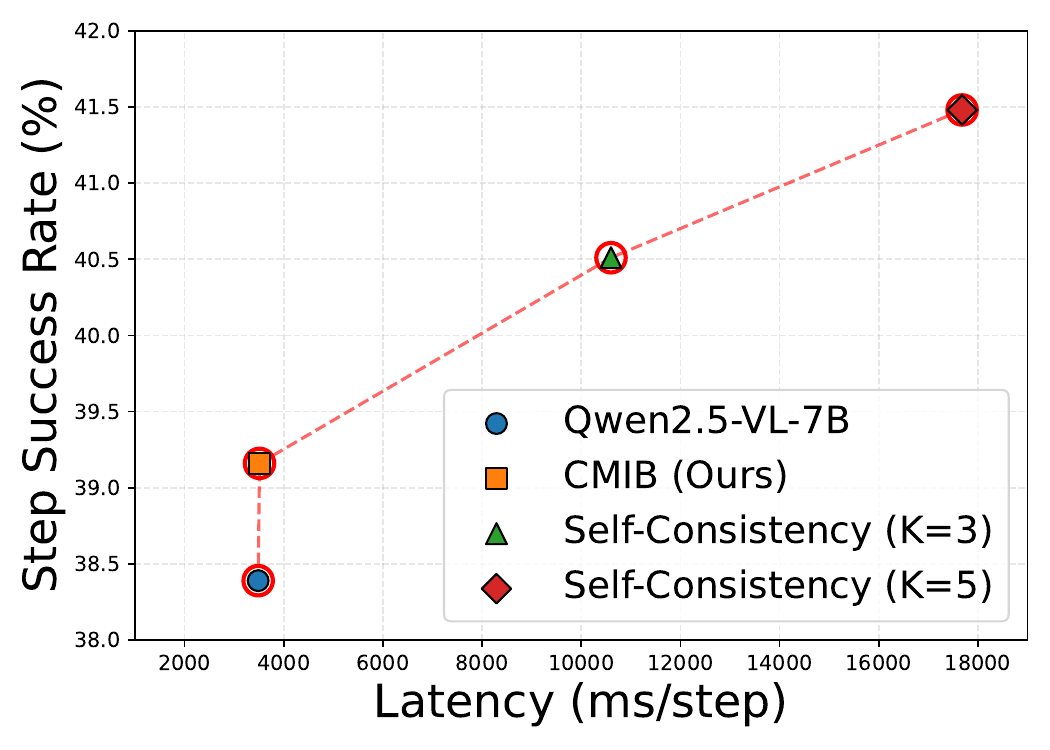}
        \vspace{-2em}
        \caption{Step success rate over computational latency.}
        \label{fig:rq4}
    \end{minipage}
    \vspace{-1.5em}
\end{figure}

%% file: content/8_conclusion.tex
\section{Conclusion}

In this paper, we introduced Skill-CMIB, a principled framework for multimodal agent skill construction that leverages the Information Bottleneck principle to enhance agent action consistency across trials. By employing a sequential decomposition, Skill-CMIB partitions skills into interpretable text-stage bottlenecks and conditional multimodal bottlenecks, effectively distilling symbolic skill cards while capturing essential residual perceptual evidence. Theoretical analysis and empirical evaluations on benchmarks such as Multimodal-Mind2Web and ScreenSpot demonstrate that our approach significantly improves task success rates and consistency without the prohibitive overhead of inference-time self-consistency. By ensuring sufficiency, minimality, and cross-modal complementarity, Skill-CMIB provides a robust foundation for building reliable multimodal agent Skill Library.

\textbf{LLM usage disclosure:} AI writing tools were used to assist in drafting and verifying the theoretical proofs in this paper. AI tools were used to assist creating the illustrative figures.

%% file: content/9_app.tex
\appendix
\onecolumn

\input{content/3_prelim}

\section{Factorization underlying CMIB}
\label{app:proof_cmib}

\begin{proof}[Proof of \Cref{lem:cmib}]
Applying the chain rule of mutual information to the first term in \Cref{eq:cmib_joint} gives
\begin{equation}
I\bigl((X,M);\,(c,z)\bigr)
=
I\bigl((X,M);c\bigr)
+
I\bigl((X,M);z\mid c\bigr).
\label{eq:cmib_chain1}
\end{equation}
Likewise, applying the chain rule to the second term gives
\begin{equation}
I\bigl((c,z);Y\bigr)
=
I(c;Y)
+
I(z;Y\mid c).
\label{eq:cmib_chain2}
\end{equation}
Substituting \Cref{eq:cmib_chain1,eq:cmib_chain2} into \Cref{eq:cmib_joint}, the objective can be written as
\begin{equation}
\Bigl[I\bigl((X,M);c\bigr)-\beta\,I(c;Y)\Bigr]
+
\Bigl[I\bigl((X,M);z\mid c\bigr)-\beta\,I(z;Y\mid c)\Bigr].
\label{eq:cmib_exact}
\end{equation}
This yields the claimed exact two-stage factorization. The generalized objective in \Cref{eq:cmib_objective} is then obtained by replacing the shared coefficient $\beta$ with stage-specific coefficients $\beta_c$ and $\beta_z$. When $\beta_c=\beta_z=\beta$, \Cref{eq:cmib_objective} reduces to \Cref{eq:cmib_exact}.
\end{proof}

\section{Variational surrogate of the conditional multimodal bottleneck}
\begin{proof}[Proof of \Cref{lem:mm_ib_variational}]
We first control the conditional compression term. Since $Z \perp X \mid (M,c^*)$, we have
\begin{equation}
I((X,M);z\mid c^*) = I(M;z\mid c^*).
\label{eq:mm_ib_proof_markov}
\end{equation}
Now define the aggregated posterior
\begin{equation}
q_\theta(z\mid c^*)
=
\int q_\theta(z\mid M,c^*)\,p(M\mid c^*)\,dM.
\label{eq:mm_ib_agg_post}
\end{equation}
Using the standard variational decomposition,
\begin{align}
&\mathbb{E}_{M\sim p(\cdot\mid c^*)}
\!\left[
\mathrm{KL}\bigl(q_\theta(z\mid M,c^*)\,\|\,r_\phi(z\mid c^*)\bigr)
\right] \nonumber\\
&\qquad=
I(M;z\mid c^*)
+
\mathrm{KL}\bigl(q_\theta(z\mid c^*)\,\|\,r_\phi(z\mid c^*)\bigr)
\;\ge\;
I(M;z\mid c^*).
\label{eq:mm_ib_proof_kl}
\end{align}
Combining \Cref{eq:mm_ib_proof_markov,eq:mm_ib_proof_kl} yields
\begin{equation}
I((X,M);z\mid c^*)
\;\le\;
\mathbb{E}_{M\sim p(\cdot\mid c^*)}
\!\left[
\mathrm{KL}\bigl(q_\theta(z\mid M,c^*)\,\|\,r_\phi(z\mid c^*)\bigr)
\right].
\label{eq:mm_ib_proof_compression}
\end{equation}

Next we lower-bound the conditional relevance term. By the definition of conditional mutual information,
\begin{equation}
I(z;Y\mid c^*) = H(Y\mid c^*) - H(Y\mid z,c^*).
\label{eq:mm_ib_proof_mi}
\end{equation}
For any predictive distribution
$\pi_{\mathrm{tsk}}\!\left(Y \mid [\,g_\omega(z);\; c^*;\; \mathcal{B}\,]\right)$,
the conditional cross-entropy upper-bounds the conditional entropy:
\begin{equation}
H(Y\mid z,c^*)
\;\le\;
\mathbb{E}_{\substack{(M,Y)\sim p(\cdot,\cdot\mid c^*)\\ z\sim q_\theta(\cdot\mid M,c^*)}}
\!\left[
-\log \pi_{\mathrm{tsk}}\!\left(
Y \mid [\,g_\omega(z);\; c^*;\; \mathcal{B}\,]
\right)
\right].
\label{eq:mm_ib_proof_entropy}
\end{equation}
Substituting \Cref{eq:mm_ib_proof_entropy} into \Cref{eq:mm_ib_proof_mi} gives
\begin{equation}
I(z;Y\mid c^*)
\;\ge\;
H(Y\mid c^*)
+
\mathbb{E}_{\substack{(M,Y)\sim p(\cdot,\cdot\mid c^*)\\ z\sim q_\theta(\cdot\mid M,c^*)}}
\!\left[
\log \pi_{\mathrm{tsk}}\!\left(
Y \mid [\,g_\omega(z);\; c^*;\; \mathcal{B}\,]
\right)
\right].
\label{eq:mm_ib_proof_relevance}
\end{equation}

Finally, combining \Cref{eq:mm_ib_proof_compression,eq:mm_ib_proof_relevance} with the definition of $\mathcal{L}_z$ in \Cref{eq:mm_ib_orig}, we obtain
\begin{align}
\mathcal{L}_{z}
&=
I((X,M);z\mid c^*)-\beta_z I(z;Y\mid c^*) \nonumber\\
&\le
\mathbb{E}_{\substack{(M,Y)\sim p(\cdot,\cdot\mid c^*)\\ z\sim q_\theta(\cdot\mid M,c^*)}}
\!\left[
\log \frac{q_\theta(z\mid M,c^*)}{r_\phi(z\mid c^*)}
-
\beta_z \log \pi_{\mathrm{tsk}}\!\left(
Y \mid [\,g_\omega(z);\; c^*;\; \mathcal{B}\,]
\right)
\right]
-\beta_z H(Y\mid c^*) \nonumber\\
&=
\mathcal{J}_{z}(\theta,\phi;c^*)-\beta_z H(Y\mid c^*),
\end{align}
which proves the claim.
\label{app:mm_ib_variational_proof}
\end{proof}

%% file: content/3_prelim.tex
\section{Preliminaries}

\subsection{Agent Skills}
Let $\pi_{\mathrm{tsk}}$ denote a frozen task LLM that interacts with an environment over multiple decision steps. 
An \emph{agent skill} is a reusable procedure $s_i \in \mathcal{S}$, represented by a skill card $c_i \in \mathcal{C}$, that specifies how to solve a class of tasks rather than a single instance \citep{li2026skillsbench,xu2026agent2,xu2026agent}. 
A \emph{skill library} is a finite collection $\mathcal{L}=\{s_i\}_{i=1}^N$ of such skills, as commonly maintained in agentic systems \citep{ling2026agent,li2026organizing}; in many settings, the library is updated across episodes or task chains rather than within a trajectory \citep{fang2025memp,forouzandeh2025learning}. 
At step $t$, the task LLM conditions on the current state
$h_t=(x_t, a_{<t}, o_{<t}, f_{<t})$,
together with an active skill subset $\mathcal{L}_t \subseteq \mathcal{L}$ selected from the larger library \citep{zheng2026skillrouter,li2026organizing,zhou2026memento}$,$ and outputs
$a_t \sim \pi_{\mathrm{tsk}}(\cdot \mid h_t,\mathcal{L}_t)$. 
The environment then returns an observation $o_t$ and feedback $f_t$, yielding the next state $h_{t+1}$. 
Iterating this process produces a rollout
\begin{equation}\label{pre:rollout}
    \tau^{(k)}=\bigl(x_t^{(k)},\,a_t^{(k)},\,o_t^{(k)},\,f_t^{(k)}\bigr)_{t=1}^{T_k},
\end{equation}
and $K$ related trial-and-error rollouts form the bundle $\mathcal{B}=\{\tau^{(k)}\}_{k=1}^K$ \citep{wang2025reinforcement,mi2026procmem,jiang2026xskill}.

\subsection{Information Bottleneck}
The Information Bottleneck (IB) principle~\citep{tishby2000information} provides a foundational framework for learning compressed representations.
Given a source random variable $W$ and a target $Y$, the IB seeks a representation $Z$ by solving
\begin{equation}\label{pre:IB}
\min_{p(z \mid w)}\; I(W;\, Z) - \beta\, I(Z;\, Y), 
\end{equation}
where $I(\cdot;\cdot)$ denotes mutual information and $\beta \geq 0$ controls the trade-off between compression $I(W;Z)$ and relevance $I(Z;Y)$. 
Setting $W = (X, M)$ and $Z = S$, with $X$ and $M$ the aggregate textual and multimodal content associated with rollout bundles (formalized below), recovers a naive multimodal skill bottleneck, 
but this ignores the heterogeneous nature of discrete text and continuous multimodal features,
offers no retrieval interface, and provides no mechanism for ensuring cross-modal complementarity.